\documentclass[11pt]{article}

\usepackage[paper=letterpaper, margin=1in]{geometry} 
\usepackage{amsmath,amssymb,amsfonts, amsbsy, amsthm}
\usepackage{mathtools}
\usepackage{latexsym}
\usepackage[usenames]{color}
\usepackage{xcolor}

\usepackage{graphics}

\usepackage{hyperref}
\usepackage{natbib}[authoryear]

\usepackage{multicol}
\usepackage{multirow}
\usepackage{enumitem}
\usepackage[utf8]{inputenc}
\usepackage[T1]{fontenc}
\usepackage[scaled]{helvet}

\usepackage[english]{babel}

\usepackage[doublespacing, nodisplayskipstretch]{setspace}
\usepackage{placeins}
\usepackage{enumitem}
\setlist[itemize]{noitemsep, topsep=0pt}
\setlist[enumerate]{noitemsep, topsep=0pt}

\renewcommand{\P}{\operatorname{P}}

\newcommand{\vertiii}[1]{{\left\vert\kern-0.25ex\left\vert\kern-0.25ex\left\vert #1 \right\vert\kern-0.25ex\right\vert\kern-0.25ex\right\vert}}

\newcommand{\Ar}{\mathcal{A}}
\newcommand{\Br}{\mathcal{B}}
\newcommand{\Cr}{\mathcal{C}}

\newcommand{\Rr}{\mathcal{R}}

\newcommand{\Tr}{\mathcal{T}}

\theoremstyle{plain}
\newtheorem{thm}{Theorem}
\newtheorem{lem}{Lemma}

\theoremstyle{remark}

\theoremstyle{definition}
\newtheorem{definition}{Definition}

\usepackage[ruled,vlined]{algorithm2e}

\SetCommentSty{algcomment}

\newcommand{\FDR}{\text{FDR}}

\newcommand{\N}{\mathrm{N}}
\newcommand{\Unif}{\mathrm{Unif}}

\newcommand{\bI}{\mathbb I}

\newcommand{\mX}{\mathcal X}

\newcommand{\mT}{\mathcal T}

\newcommand{\mA}{\mathcal A}
\newcommand{\mB}{\mathcal B}
\newcommand{\mC}{\mathcal C}

\newcommand{\mH}{\mathcal H}

\newcommand{\mQ}{\mathcal Q}
\newcommand{\mR}{\mathcal R}

\newcommand{\mU}{\mathcal U}

\newcommand{\csil}{{|S|}}
\newcommand{\sumnull}{{\sum_{S\in\mB_0^{(\ell)}}}}

\newcommand{\J}{\lceil(\gamma_m-\bar\Phi^{-1}(\alpha))/ \nu_m \rceil}

\newcommand{\node}{\mathrm{node}}

\newcommand{\el}{{(\ell)}}

\newcommand\Item[1][]{%
  \ifx\relax#1\relax  \item \else \item[#1] \fi
  \abovedisplayskip=0pt\abovedisplayshortskip=0pt~\vspace*{-\baselineskip}}

\makeatletter
\newcommand{\mathleft}{\@fleqntrue\@mathmargin0pt}
\newcommand{\mathcenter}{\@fleqnfalse}
\makeatother

\graphicspath{{./figures/}} 

\title{DART2: a robust multiple testing method to smartly leverage helpful or misleading ancillary information}

\author{Jichun Xie and Xuechan Li}

\begin{document}

\maketitle

\begin{abstract}

    In many applications of multiple testing, ancillary information is available, reflecting the hypothesis null or alternative status. Several methods have been developed to leverage this ancillary information to enhance testing power, typically requiring the ancillary information is helpful enough to ensure favorable performance. 
    In this paper, we develop a robust and effective distance-assisted multiple testing procedure named DART2, designed to be powerful and robust regardless of the quality of ancillary information. When the ancillary information is helpful, DART2 can asymptotically control FDR while improving power; otherwise, DART2 can still control FDR and maintain power at least as high as ignoring the ancillary information. 
    We demonstrated DART2's superior performance compared to existing methods through numerical studies under various settings. In addition, DART2 has been applied to a gene association study where we have shown its superior accuracy and robustness under two different types of ancillary information.
    
\end{abstract}    


\newpage

\section{Introduction}

Multiple testing is a useful statistical framework that has been widely applied to genomics, clinical trials, neuroimaging, and environmental studies. In many studies in these fields, ancillary variables are available to provide contextual information about the hypothesis. One commonly seen assumption is that these ancillary information, such as genetic annotations, patient demographic variables in clinical trials, and neuron spatial locations in neuroimaging studies imply the null or alternative statuses of the hypotheses regarding genetic variants, patients, or brain neurons. 

Under the assumption, many methods have been developed to leverage the ancillary information to improve testing accuracy. A popular approach is covariate-adaptive testing, which convert ancillary information into covariates linked to each hypothesis. Those methods gain power by incorporating covariates through parametrically modeling the distribution of testing statistics or the proportion of null hypotheses \citep{leung2022zap,lei2018adapt,yun2022detection,cai2022laws,qiu2021neurt}. Another category includes methods that employ non-parametric models to categorize hypotheses into groups based on their ancillary information, subsequently enhancing power through the aggregation of p-values \citep{zhang2011multiple,li2023dart} or the application of weighted p-values within these groups \citep{hu2010false,barber2017p,yang20242dgbh}. Although these methods can improve testing accuracy when the ancillary information correctly implies the hypothesis status following their proposed models, they may suffer from inflated false discovery rates (FDR) or reduced power when this assumption does not hold. Especially, many methods require specifying correct parametric models or priors, which can be challenging in practice, and thus greatly compromise the performance of the methods in real data analysis.

To address this challenge, we introduce DART2, which does not rely on accurately specifying the ancillary information; it improves the FDR control and power of multiple testing at different levels of ancillary information quality. DART2 has two stages: screening and refining. The screening stage is similar to DART \citep{li2023dart}; it hierarchically aggregates the hypotheses based on their ancillary information and test them at different resolution levels; the screened out hypotheses will be recorded. If we stop at this stage, when the ancillary information does not correctly reflect the hypothesis status, the testing results may have inflated FDR. Thus, DART2 adds the refining stage to examine all the screened-out hypotheses again and selectively reject them based on their testing statistics. This selection strategy ensures that the FDR control is guaranteed even when the ancillary information is not accurate.

Compared to the existing methods, DART2 has two major benefits:
\begin{itemize}
\item \textit{Robust and powerful regardless of the quality of ancillary information:} DART2's unique refining stage guarantees the FDR control even when the ancillary information does not correctly reflect the hypothesis status. DART2 ensures asymptotic type I error control even if the ancillary information is wrong, and its power at least matches Benjamini-Hochberg (BH) procedure\citep{benjamini1995controlling}, thereby providing a robust and effective tool for hypothesis testing. On the other hand, when ancillary information partially reflect the hypothesis status, DART2 can improve power, matching the performance of other ancillary information-based methods.
\item \textit{Fewer constraints on the format of ancillary information:}
Most existing methods demand that ancillary information be explicitly linked to each hypothesis individually. In contrast, our approach relies on ancillary information structured as an aggregation tree, requiring only the mutual information between hypotheses. In such an aggregation tree, hypotheses with greater similarities are clustered together, sharing more common ancestor nodes, eliminating the need for information to be attached to each hypothesis directly. More details about the aggregation tree can be found in section~\ref{sec:method}.
\end{itemize}

The rest of the paper is organized as follows. Section~\ref{sec:method} introduces the DART2 method and its theoretical properties. Section~\ref{sec:simu} presents simulation studies that compare DART2 with other competing methods. Section~\ref{sec:real-data} demonstrates DART2's superior performance in a gene association study compared to the competing methods. Section~\ref{sec:discussion} discusses the results and concludes the paper.


\section{Methods}\label{sec:method}

\subsection{Hypotheses and aggregation tree}\label{sec:tree}

Suppose $m$ null hypotheses form the null set  $\Omega_0$ and $m_1$ alternative hypothesis form the alternative set $\Omega_1$, with $\Omega_0\cap \Omega_1=\emptyset$, and $|\Omega_0\cup \Omega_1|= m_0+m_1= m$. For any set $\Ar$, the notation $|\Ar|$ represents its cardinality.

Denote the test statistic for hypothesis $i$ by $T_i$. We assume $T_i$ follows $N(\theta_i, 1)$  with $\theta_i\leq 0$ if $i\in \Omega_0$. For example, Wald tests usually generate statistics asymptotically following these distributions under the null. More generally, if P-values $P_i$ are summarized for tests, we can transform them to test statistics $T_i$ by $T_i=\Phi^{-1}(1-P_i)$, where $\Phi^{-1}(\cdot)$ is the inverse standard Gaussian distribution function. The alternative distribution of $T_i$ could be flexible, although under traditional 
hypothesis testing settings, we often assume 
\[\P(T_i < x \mid i\in \Omega_0) \geq \P(T_i <x \mid i\in\Omega_1).\] 
In fact, the distribution requirement of $T_i$ can be further relaxed. For example, the null distribution of the null P-value statistics $P_i$ might be asymptotically super-uniform. See \citet{li2023dart} for more details. We focus on less general cases in this paper for keeping the asymptotic results simple.

To implement DART2, we need an aggregation tree $\mT$ to summarize the ancillary information.  \citet{li2023dart} introduces an algorithm to construct $\mT$ based on the distances between hypotheses: a shorter distance between two hypotheses suggest that they are more likely to be co-null or co-alternative. After applying the tree construction algorithm, let the resulting aggregation tree be $\mT=\{\mA^{(1)},\ldots,\mA^{(L)}\}$.
$\mA^{\el}$ is the node set of level $\ell$ of the tree; each node is a set of hypotheses that will be tested on this layer.
$\mA^{(1)}$ is the leaf node set, containing individual hypotheses $\mA^{(1)}=\{\{1\},\ldots,\{m\}\}$. On a higher layer $\mA^\el$ $(\ell\geq 2)$, close nodes from $\mA^{(\ell-1)}$ are merged to form nodes in $\mA^\el$. For any $
S\in \mA^\el$, its children set is denoted by $\mC(S)$, a subset of $\mA^{(\ell-1)}$.

Although any finite layer and finite children aggregation tree can be used to implement DART2, proper tuning parameter setting might help achieve better results for DART2. These parameters can be implemented with the tree construction algorithm in \citet{li2023dart}.
\begin{itemize}
  \item \textit{Maximum children size $M$}: 
  We recommend setting $M=2$. Larger values of $M$ might lead to some nodes containing many hypotheses, which need to be re-examined in the refining stage. Too many large nodes will decrease the statistical and computational efficiency of the refining stage.
  \item \textit{Maximum layer number $L$:}
  Our approach mirrors that of DART in establishing the maximum number of layers, defined by $
  L=\lfloor\log_Mm-\log_Mc_m\rfloor$, with $c_m$ indicating the preferred number of nodes at layer $L$. Previously, DART \citep{li2023dart} suggested using $c_m\geq 35$ to ensure asymptotic FDR control. Because DART2 is a more robust algorithm, we can lower $c_m$ to 5. In all numerical and empirical experiments in this paper, we set $c_m=5$ to get the maximum number of layers and thus maximum power. Besides, The robustness of DART2 across various $L$ values is examined in section~\ref{sec:simu}.
  \end{itemize}


Other than being constructed based on the distance matrix, the aggregation tree can also be obtained from prior knowledge, such as Phylogenetic trees and genealogy trees. A phylogenetic tree illustrates the evolutionary connections between biological entities (e.g. species or taxa); and a genealogy tree shows the inheritance relationships between family members. In both examples, closer nodes on the tree indicate the entities share a more recent common ancestor, suggesting their possible similarities in their physical or genetic characteristics. If each biological entity or family member on the tree leaves forms a hypothesis regarding its function or phenotype, then these trees can be used to implement DART2.

\subsection{Algorithm}

DART2 has two stages: screening and refining. In the screening stage, DART2 go through the aggregation tree to screen out hypothesis sets (nodes) that might contain alternative hypotheses. In the refining stage, DART2 examines all screened-out hypotheses to select a subset that is more likely to be alternative.

DART2's algorithm is outlined in Algorithm~\ref{alg:screen}, with details explained afterwards. A toy example with $7$ hypotheses arranged under a predefined $3$-layer aggregation tree is illustrated in Figure~\ref{fig:method}. 
\par

\begin{algorithm}[ht]
    \KwData{Test statistics $T_1,\ldots,T_m$; tree $\Tr=\{\Ar^\el: \ell\in [L]\}$.}
    \KwResult{Rejected hypotheses $\Rr$, $\ell\in[L]$.}
    \vspace{0.5cm}
    \tcp{Screening stage}
    On layer 1 (leaf layer), set the \textbf{node screening threshold} $\hat{c}^{(1)}$ as \eqref{eq:t1}\;
    Initialize the rejection set $\Rr=\{i:T_i>\hat{c}^{(1)}\}$\;
    $\Rr_\node=\{\{i\}: T_i>\hat{c}^{(1)}\}$\tcp*{Initialize the screened-out node set}
    \For{$\ell \in \{2,\ldots,L\}$}{
    $\tilde\mA^\el=\{A\setminus R: A\in\mA^\el,R=\{i:i\in S, S\in \mR_\node\}\}$\tcp*{Remove the rejected hypothesis from all nodes}
        Define the qualified node set $\Br^\el=\{S: S\in \tilde A^\el, |\Cr(S)|
        \geq 2\}$\;
        Derive the node testing statistics $T_S$ as in \eqref{eq:ts} for each $S\in \Br^\el$\;         
        Set the \textbf{node screening threshold} $\hat{c}^\el$ as in \eqref{eq:tl}\; Update the screened-out node set $\Rr_\node=\Rr_\node\cup\{S: T_{S} > \hat{c}^\el\}$.
    }
    \vspace{0.5cm}
    \tcp{Refining stage}
    For all $S\in \Rr_\node$, set the \textbf{refining threshold} $\hat{t}_S$ as in s\eqref{eq:rrt}\;
    Update the rejection set as $\Rr=\Rr\cup\{i: i \in \cup_{S\in \Rr_\node} S, \ T_i \geq \hat{t}_S\}$
    \caption{DART2 procedure}
    \label{alg:screen}
\end{algorithm}
\par
\begin{figure}[!htpb]
\includegraphics[width=\textwidth]{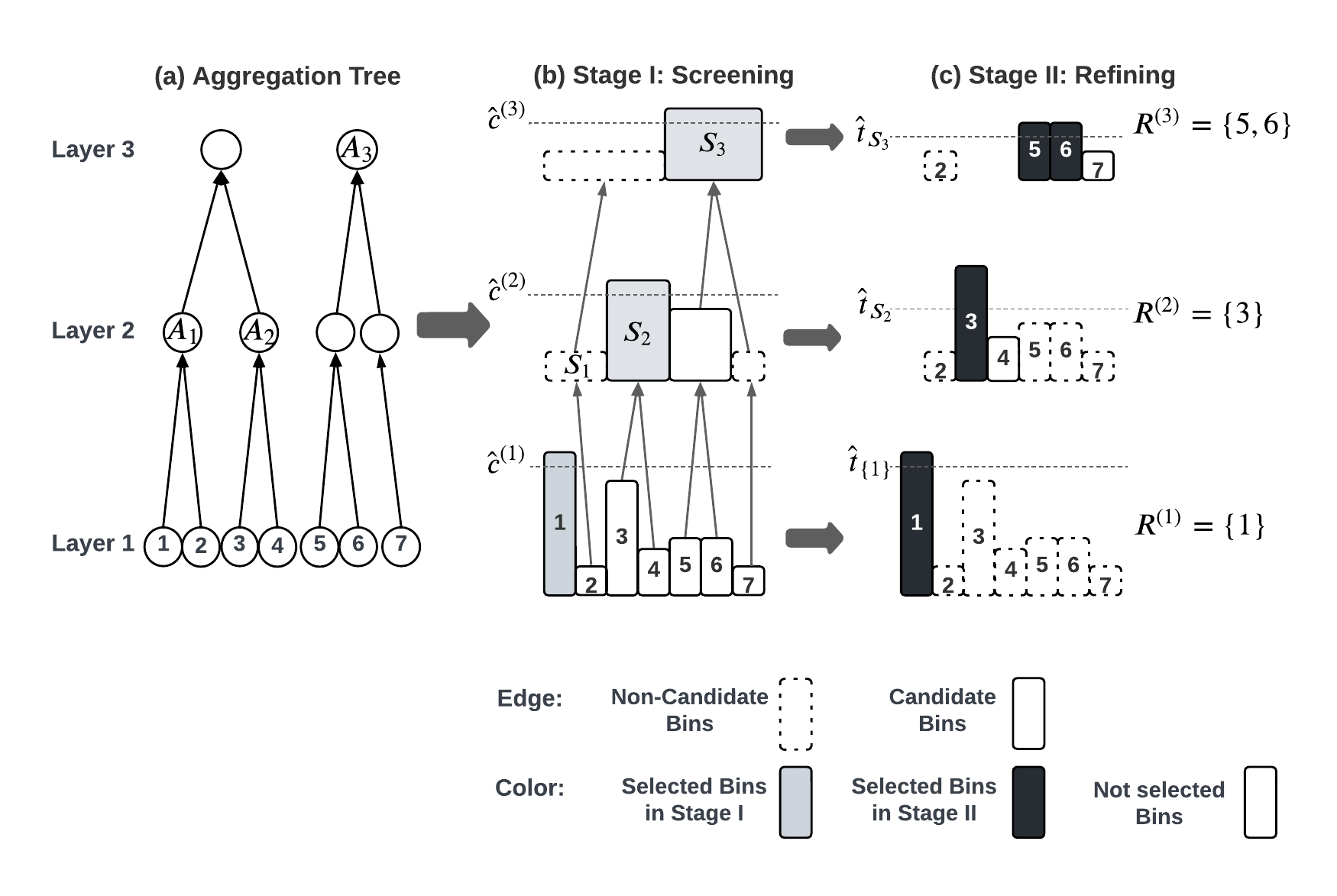}
\caption{An illustration example of DART2 procedure with $7$ features. (a) An aggregation tree obtained from prior knowledge, comprising $L=3$ layers with a maximum of two children allowed for each node ($M=2$); (b) Screening process embed in the aggregation tree, where each parent node is associated with a node-level hypothesis. Nodes on the aggregation tree are depicted as bins, with higher bins corresponding to nodes with a larger test statistics $T_{S}$; (c) 
Refining process for further selecting the features located within the nodes (bins) that were screened-out in Stage I. The rejected hypotheses on each layer is presented as $R^\el$ in the figure, and the final rejection set is $\mR=\{1,3,5,6\}$.}
\label{fig:method}
\end{figure}

\subsubsection*{Screening stage}

In the screening stage, DART2 will test all nodes (hypothesis sets) in the aggregation tree from the leaf layer to the top layer. The purpose is to screen some hypothesis sets that probably contain one or more alternative hypotheses. Here we introduce the technical details.

For any node $S$ on the aggregation tree, we define the node-level hypothesis as follows:
\begin{equation}\label{eq:nodehypo}
H_{0,S}: \text{ for all } j\in S, j\in\Omega_0 \text{ versus } H_{A,S}: \text{ there exists some } j\in S \text{ such that } j\in\Omega_1.
\end{equation}
Therefore, a screened-out node indicates strong evidence that it contains at least one alternative hypothesis. On layer 1, each node is a singleton set, and the node-level hypothesis is equivalent to the individual hypothesis. Thus, a screened-out node on layer 1 corresponds to a rejected hypothesis.

DART2's layer 1 node threshold is same as DART:
\begin{equation}
\label{eq:t1}
    \hat{c}^{(1)} = \inf \left\{
        \alpha_m \leq c \leq \alpha: \frac{\sum_{S\in\mB^\el}|S|\bar{\Phi}(c)}{\max\{\sum_{S\in\mB^\el}|S|I\{T_S>c\},1\}} \leq \alpha
    \right\},
\end{equation}
where $\alpha_m = (m \log m)^{-1}$, and $\bar{\Phi}(\cdot)$ is the complementary cumulative density function of standard Gaussian, defined as $\bar{\Phi}(\cdot ) = 1-\Phi(\cdot)$. This threshold has been used in other multiple testing literature \citep{liu2013gaussian,xie2018false}.

On layer $\ell\ (\ell\geq 2)$, we recursively define $\tilde{\mA}^\el$ as the node set on layer $\ell$ after removing the rejected hypotheses from the previous layer. The qualified node set $\mB^\el$ is defined as the set of nodes with at least two children; all nodes in $\mB^\el$ will be tested on layer $\ell$. We exclude those nodes with only one child because they have been tested before layer $\ell$. 

For any node $S\in \mB^\el$, we define the node-level statistic $T_S$ using the Stouffer aggregation \citep{stouffer1949american} for testing the node-level hypothesis \eqref{eq:nodehypo}:
\begin{equation}
  T_S=\sum_{i\in S}T_i/\sqrt{|S|}.
  \label{eq:ts}
\end{equation}
The node screening threshold $\hat{t}^\el$ is
\begin{equation}
  \label{eq:tl}
      \hat{c}^\el = \inf \left\{
          \alpha_m \leq c \leq \alpha^\el: \frac{\sum_{S\in\mB^\el}|S|\bar{\Phi}(c)}{\max\{\sum_{S\in\mB^\el}|S|I\{T_S>c\},1\}} \leq \alpha^\el
      \right\},
  \end{equation}
where $\alpha^\el$ is the layer-specific node FDR control level. Although in theory setting $\alpha^\el=\alpha$ still guarantees DART2's FDR control (see proof of theorem~\ref{thm:indfdr} in \ref{app:proof_thm}), numerical studies suggest that optimal performance in finite samples is achieved through a more conservative approach. Specifically, we recommend setting 
$\alpha^\el=\alpha/\max\{|S|:S\in\mB^\el\}$ for best finite sample performance suggested by numerical studies.

\subsubsection*{Refining stage}

If the ancillary information does not reflect hypothesis status the aggregation tree might fail to aggregate co-null or co-alternative hypotheses in the same node. The rejections of these nodes might only indicate that the node contains at least one alternative hypothesis, not that all hypotheses in the node are alternative. Therefore, we further define a refining threshold; among all hypotheses in the screened-out nodes, only those with testing statistics exceeding this threshold will be rejected.

For any screened-out node $S\in \Rr^\el$, 
\[T_S\geq \hat c^\el \Longleftrightarrow \sum_{i\in S}T_i/|S|\geq \hat c^\el/\sqrt{|S|}.\]
Thus, individual hypotheses with $T_i \geq \hat c^\el/\sqrt{|S|}$ contribute more to the node rejection. Intuitively, these hypotheses are more likely to be alternative. Based on this idea, we propose a naive refining threshold $\hat t^*_S=\hat c^\el/\sqrt{|S|}$ for any screened-out node $S\in\mB^\el$, and only reject hypotheses with $T_i\geq \hat t^*_S$.
\begin{thm}\label{thm:indfdr}
  Assume the number of alternative hypothesis $m_1=O(m^{r_1})$ for some
    $r_1<(M^{L-1}+1)^{-1}$, DART2 with naive refining threshold $\hat t^*_S$ controls the FDR at any pre-specified level $\alpha\in (0,1)$, i.e., $\lim_{m,n\to\infty}\FDR \leq \alpha$.
  \end{thm}
Although Theorem~\ref{thm:indfdr} shows the Naive refining threshold $\hat t^*_S$ provides asymptotic validity in theory, we found that it is too liberal in practice. Numerical studies suggest that the naive method might lead to moderate FDR inflation under finite sample and hypothesis case. One possible reason of the inflation is that, higher-layer nodes contain many hypotheses and $\hat t^*_S$ may be too small to exclude null hypotheses. To address this, a more robust and conservative refining threshold is proposed.

\begin{definition}[Robust refining threshold]
  \label{def:rrt}
  For a screened-out node $S\in\mB^\el$,
  its robust refining threshold is defined as 
  \begin{equation}
  \label{eq:rrt}
  \hat t_S=\min\{\hat t_{S,1},\hat t_{S,2}\}\text{, with } \hat t_{S,1}=\max\{\hat t^*_S,\bar\Phi^{-1}(\alpha)\}\text{ and } \hat t_{S,2}=\max\{T_i,i\in S\}.
  \end{equation}
  \end{definition}

  The robust refining threshold introduces the lower bounds $\bar\Phi^{-1}(\alpha)$ in $\hat t_{S,1}$ to avoid excessively small thresholds, where $\bar\Phi(\cdot)$ is the complement cumulative density function of the standard Gaussian. Additionally, because the screened-out nodes is considered to contain at least one alternative hypothesis, we set $\hat t_{S,2}$ to ensure at least one hypothesis will be rejected.
  The following theorem ensures the asymptotic FDR control by applying the robust refining threshold.
\begin{thm}\label{thm:robfdr}
Assume the number of alternative hypothesis $m_1=O(m^{r_1})$ for some
  $r_1<(M^{L-1}+1)^{-1}$, DART2 with robust refining threshold controls the FDR at any pre-specified level $\alpha\in (0,1)$, i.e., $\lim_{m,n\to\infty}\FDR \leq \alpha$.
\end{thm}


\section{Numerical Experiments} \label{sec:simu}

We set up $m=1000$ hypotheses. The hypothesis status is determined by $\{\eta_i: i\in[m]\}$, which are independently generated by the following rules:
\begin{equation*}
    \eta_{i} =\Big\{\{[3.4\phi_1(d_{156,i})-0.6]\vee 0\}+3\phi_2(d_{800,i})-0.1\Big\}\vee 0, 
\end{equation*}
where $\phi_1$ and $\phi_2$ are the probability density functions of
$\N(0,1)$ and $\N(0,0.1)$, respectively. If $\eta_i=0$, then hypothesis $i$ is null; otherwise, it is alternative. After fixing the random seed to generate $\eta_i$, we obtained 216 alternative hypotheses and the rest are null. Our aim is to test these hypotheses with a desired FDR $\alpha \in \{1\%, 5\%\}$.

Suppose all hypotheses are in a two-dimensional Euclidean space. Under the ideal case, shorter distances between hypotheses imply a higher likelihood of two hypotheses' co-null or co-alternative status. Especially, for our numerical experiments, the alternative hypotheses are concentrated within two clusters, as shown in Appendix Figure \ref{appfig:eta}. Thus, under the ideal case, the distances can be treated as helpful ancillary information to improve testing accuracy. 

To evaluate method robustness, we also considered non-ideal scenario where the ancillary information is less helpful. We introduced a misleading level $\tau \in [0,1]$ to represent the proportion of alternative hypotheses randomly switched with null hypotheses. Thus, $\tau=0$ representing fully informative ancillary information and $\tau=1$ indicating completely non-informative (misleading) ancillary information. We varied $\tau\in\{0,0.2,0.4,0.6,0.8,1.0\}$ to systematically evaluate and compare DART2's performance against other methods.

Regardless of the helpfulness of the ancillary information, we applied the aggregation tree construction algorithm in DART\citep{li2023dart} to build the aggregation tree. The maximum number of children for each node is set as $M=2$, and by setting $c_m=5$, the maximum number of layers is $L=7$.

We used four settings to generate test statistics, denoted as SE1–SE3. SE1 simulated test statistics following Gaussian distribution. SE2 simulated test statistics with intentionally misspecified null distributions to assess robustness. SE3 and SE4 correspond to Wald test statistics from a linear regression model and the Cox proportional hazard model, respectively. The details of the simulation settings are provided in \ref{app:simu}.
Each simulation for a specific test statistic type comprises 200 repetitions.

\subsubsection*{Evaluating DART2's validity and robustness}

We applied DART2 to all settings and evaluated its performance when the aggregation tree layer number $L$ and the ancillary information misleading level $\tau$. The performance is measured by average false discovery proportion  and sensitivity over the 200 repetitions. If DART2 stopped at layer 1, then this is asymptotically equivalent to the Benjamini-Hochberg (BH) procedure\citep{benjamini1995controlling}. The results are shown in Figure \ref{fig:layerwise}. 

Based on the number of hypotheses $m=1000$, we reach the root of the aggregation tree when $L=13$; thus, $13$ is the maximum possible number of layers in the aggregation tree. Across all $L$ values, DART2 consistently controls the FDR under the desired level $\alpha$. This indicates that increasing testing layers will not inflate the FDR. The testing sensitivity increases with $L$ when $L$ reaches $7$ and keeps stable afterwards. This indicates that DART2 has already reached the maximum power at $L=7$. To save computation time, there is no need to further increase the layer number. This fits our suggested aggregation tree construction criteria in Section~\ref{sec:tree}. 

Across all misleading level $\tau$, under all cases, DART2 controls the average FDP within the desired level $\alpha$. Besides,
as $\tau$ increases, DART2's sensitivity decreases a little, but still maintains a high level. Specifically, when $\tau=1$, DART2's power still increases as $L$ increases; suggesting DART2's superior performance than the BH procedure even when the ancillary information is completely misleading. 
This is because DART2 uses an aggregation structure to test aggregated node hypotheses; even if the ancillary information does not help to construct the informative aggregation tree, DART2 may still identify some additional true alternatives by chance.

\begin{figure}[!htpb]
    \includegraphics[width=\textwidth]{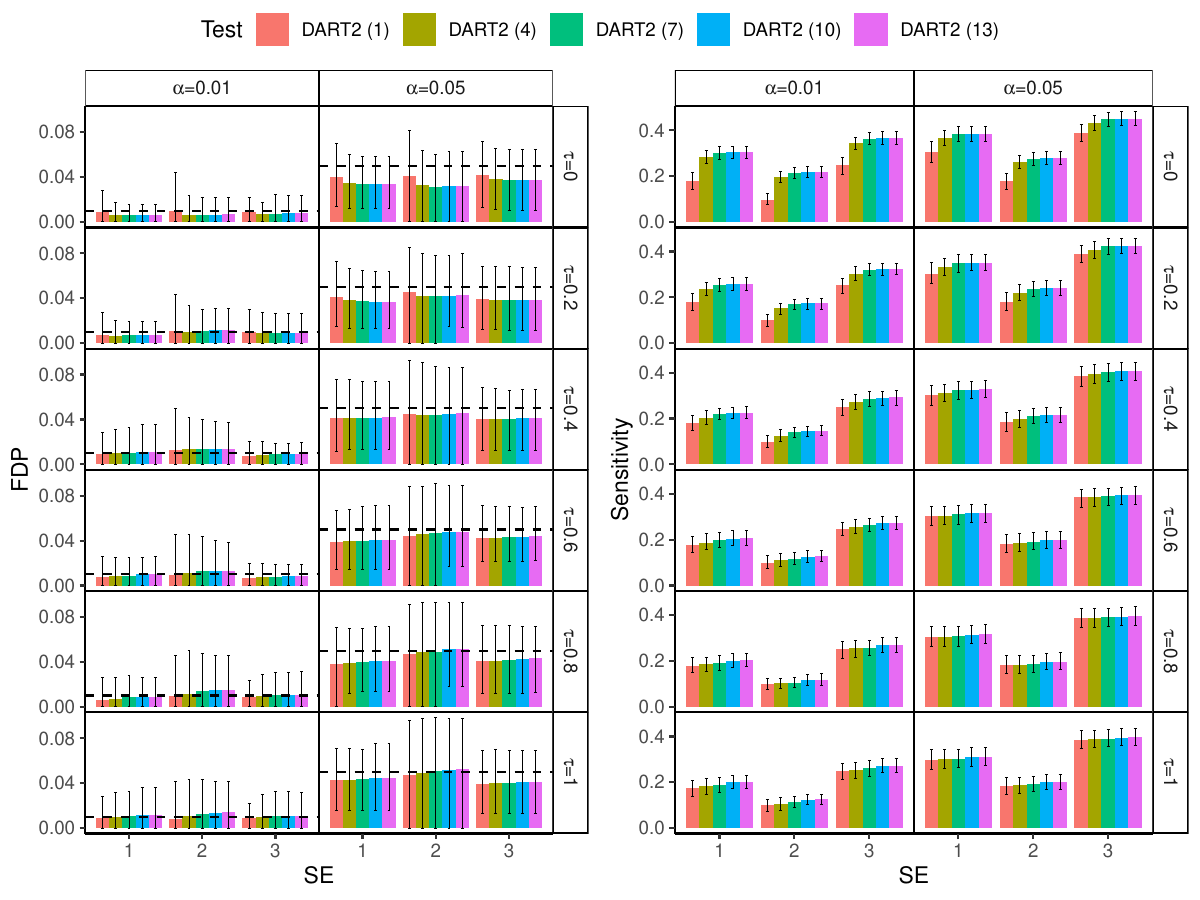}
    \caption{Performance of DART2 was evaluated across varying numbers of layers, with desired feature-level FDR $\alpha \in \{1\%, 5\%\}$, and misleading level $\tau \in \{0,0.2,0.4,0.6,0.8,1\}$. The bars indicate DART2's average performance (FDP and sensitivity) in testing $m=1000$ hypotheses, while the error bars the $90\%$ confidence intervals (the $5\%$ and $95\%$ quantiles) over the $200$ repetitions. The left panel shows the average FDP, with dashed horizontal lines indicating the desired FDR level $\alpha$. The right panel shows the average sensitivity. }
    \label{fig:layerwise}
\end{figure}

\subsubsection*{Comparing with Competing methods}

We compare the performance of DART2 with four different methods: BH, DART, AdaPT \citep{lei2018adapt}, and FDR$_L$ \citep{zhang2011multiple}. AdaPT is an iterative FDR control procedure that incorporates side information into testing. It proposed an EM algorithm to parametrically estimate the proportion of nulls and the densities of P-values to implement the method. Meanwhile, FDR$_L$ incorporates side information by aggregating p-values from each hypothesis’s $k$ nearest neighbors, using these aggregated p-values to conduct multiple testing. To simplify comparison, for DART2, we used the aggregation tree with $M=2$ and $L=7$ following the optimal tree construction criteria in Section~\ref{sec:tree}. Figure~\ref{fig:methods} shows the average FDP and sensitivity of each method. The FDP and sensitivity error bars are shown in Appendix Figure~\ref{fig:methods_bar}.

As expected, BH procedure's performance is almost the same as DART2 with $L=1$. It can successfully control the desired FDR no matter whether the ancillary information is helpful or not; however, BH's sensitivity is too low when the ancillary information is helpful or partially helpful. This is because BH completely ignores the ancillary information. 

Two competing methods, DART and FDRL maintained higher sensitivity while controlling average FDP at the desired level when $\tau=0$. However, their average FDP inflated as $\tau$ increases: DART has a moderate inflation, and FDRL has the largest inflation. Also, we observed a wide $90\%$ FDP confidence interval for DART and FDRL (Figure~\ref{fig:methods_bar}). Since the simulation under $\tau>0$ involves randomly switching the null and alternative hypotheses under each repetition, the wider confidence intervals suggest that their testing results are not robust, highly depending on the ancillary information even if the misleading level is the same. 

AdaPT effectively controls the average FDP under the desired level, even when the ancillary information is misleading. When $\alpha=5\%$, AdaPT's sensitivity is higher than DART2 when $\tau\leq 0.2$. However, when $\alpha=1\%$, AdaPT's sensitivity is close to 0, significantly worse than DART2. This indicates that AdaPT's performance relies on the quality of ancillary information, and it fails to identify true alternatives when FDR control criterion is stringent.

DART2 outperforms all competing methods in terms of both average FDP and sensitivity. DART2 maintains the average FDP within the desired level $\alpha$ and has the highest sensitivity across all settings. This indicates that DART2 is robust and powerful regardless of the quality of ancillary information and the desired FDR level.

DART2 also demonstrates significantly higher computational efficiency. Across various simulation settings and violation levels, the median computation time for a single repetition using DART2 is 17.7 seconds. In contrast, AdaPT has $11.7\%$ runs that fail to converge within 1 hour; among those that successfully converge within 1 hour, the median computation time is 540.9 seconds. This is approximately 30 times longer than that required for DART2. More details for the median computation time and the average converging percentage by simulation settings and misleading levels are presented in Table~\ref{tab:time} and Table~\ref{tab:fails}.


\begin{table}[ht]
\centering
\caption{Mean and standard deviation of the computation time for DART2 and AdaPT*}
\label{tab:time}
\begin{tabular}{cccccccc}
  \hline
  \multirow{2}{*}{$\tau$}&\multicolumn{3}{c}{AdaPT}&&\multicolumn{3}{c}{DART2}\\ \cline{2-4}\cline{6-8}
& SE=1 & SE=2 & SE=3 && SE=1 & SE=2 & SE=3\\ 
 \hline
0.00 & 941.1 (891.4) & 807.6 (725) & 1113.3 (993.9) && 18.1 (2.8) & 12 (2.1) & 23.1 (3.2) \\ 
  0.20 & 539.3 (542.4) & 550.7 (552.7) & 480.7 (392.7)&  & 18 (2.8) & 12 (2.1) & 23.2 (3.1) \\ 
  0.40 & 465.6 (282.2) & 467.3 (294.2) & 455.4 (232.3)&  & 17.9 (2.8) & 12 (2) & 23.2 (3.1) \\ 
  0.60 & 466.4 (215.1) & 442.5 (78.7) & 456.8 (68.6)&  & 17.7 (2.8) & 11.9 (2.1) & 23 (3.2) \\ 
  0.80 & 448.9 (44.6) & 437.5 (45.7) & 440.9 (36.3)&  & 17.9 (2.8) & 12.1 (2.1) & 23.1 (2.9) \\ 
  1.00 & 439.1 (42.7) & 471.3 (256.4) & 439.6 (184.6)&  & 17.9 (2.9) & 11.5 (2.1) & 21.8 (3.5) \\   \hline
\end{tabular}
\end{table}
\footnotesize{*The mean and standard deviation are presented in the format of mean (standard deviation). The computation time for AdaPT is calculated based only on those repetitions that successfully converged within 1 hour.}

\begin{figure}[!htpb]
    \includegraphics[width=\textwidth]{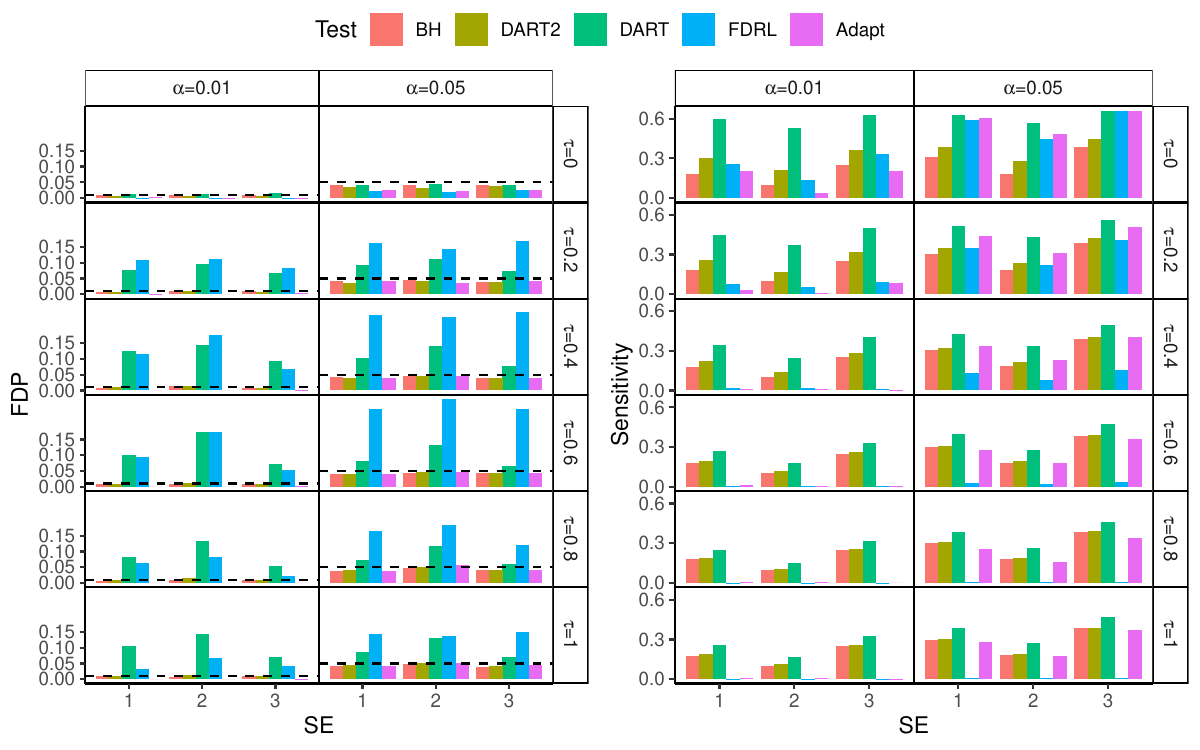}
    \caption{Performance comparison of the 7-layer DART2 with the competing method under different types of testing statistics, different desired FDR level $\alpha$ and different misleading level $\tau$. The primary bars indicate the average performance over $200$ repetitions.
    The left panel shows the average feature-level FDP, with dashed horizontal lines indicating the desired FDR $\alpha$. The right panel shows the average feature-level sensitivity.}
    \label{fig:methods}
\end{figure}


\section{Real Data Application}
    \label{sec:real-data}

We applied DART2 to a breast cancer study, with pre-analyzed P-values to indicate the significance of differential gene expressions in response to estrogen in breast cancer cells. The data is originally from the NCBI Gene Expression Omnibus (GEO) database with the access number GSE4668. It can also be accessed through the 'GEOquery' package \citep{davis2007geoquery}. The dataset contains $m=22,283$ genes in response to estrogen treatments in breast
cancer cells with different dosage levels. The objective is to identify the genes that respond to a low dosage of treatment. The analysis involves comparing gene expression levels between a control group (receiving a placebo) and a low-dose treatment group, using p-values to assess the evidence for changes in gene expression. 
The processed data includes two distinct sets of gene orderings: 1) a highly informative ordering obtained based on the
data from patients receiving higher dosage level of treatment, with genes ordered based on the response level to these higher dosages, and 2) a median informative ordering obtained based on the data from patients receiving intermediate dosage level treatment, organizing genes based on their response level to this intermediate dose level. Previously, \citet{lei2018adapt} analyzed this dataset and derived P-values and orderings for each gene to indicate if their gene expressions have been significantly altered by estrogen. 

Utilizing the highly informative ordering and the median informative ordering, we designed a study to check the performance of all methods when the ancillary information is partially informative. First, We set up the hypotheses for each gene to test if their differentially expressed.  Next, we applied DART2 and four competing methods BH, DART, AdaPT, and FDRL to the P-value statistics together with the highly informative ordering as the ancillary information with the desired FDR $\alpha=5\%$. The genes selected by at least two methods are set as the benchmark. Subsequently, we implement all methods with the median informative ordering, and then compare each method's testing results with the benchmark. We used sensitivity and precision to measure the performance.
\begin{align*}
    &\text{Sensitivity}=\frac{\text{Number of genes identified also included in benchmark} }{\text{Number of genes included in benchmark}}\\
    &\text{Precision}=\frac{\text{Number of genes identified also included in benchmark} }{\text{Number of genes identified}}
\end{align*}

Figure~\ref{fig:contour} illustrates the performance for all methods except for the BH procedure. The BH procedure is not shown since it failed to identify any genes. To facilitate a clearer comparison, the F1 score contours are included. The F1 score represents the harmonic mean of precision and sensitivity, thus serving as an overall accuracy metric: a high F1 score indicates high accuracy. FDRL is characterized by low sensitivity and precision, while DART is noted for its high sensitivity but significantly low precision. AdaPT stands out for its high precision, albeit with lower sensitivity. DART2, on the other hand, achieves the highest F1 score among all methods, indicating its superior overall performance. These results align well with our simulation studies, further demonstrating DART2's robustness and power in real data applications.
 
\begin{figure}[!htpb]
    \includegraphics[width=\textwidth]{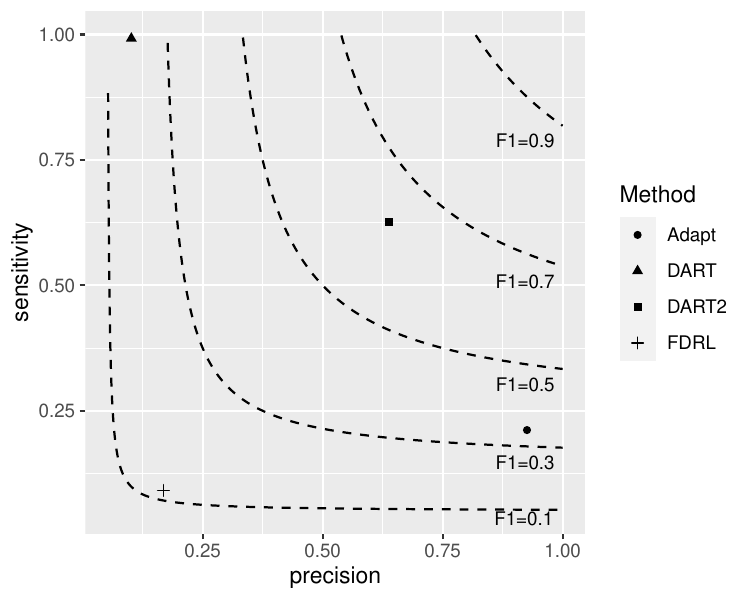}
    \label{fig:contour}
    \caption{F1 score contour plots to compare the performance of DART2, AdaPT, DART and FDRL. The dashed line represents the F1 score contour.}
\end{figure}
            
    
\section{Discussion}\label{sec:discussion}

In this paper, we introduced DART2, a novel two-stage robust hypothesis testing method to automatically adapt to the quality of ancillary information. 
DART2's stage 1 is very similar to DART. It hierarchically aggregates hypotheses based on their ancillary information and tests them at different resolution levels. The major difference between DART and DART2 is that DART2 has an additional stage 2, the refining stage, which ensures FDR control by selectively rejecting hypotheses in the screened-out hypothesis set. In stage 2, each individual hypothesis is tested again separately, to ensure the hypothesis level asymptotic FDR level control. Thus, when ancillary information is informative, it will improve power while maintaining asymptotic FDR control. When ancillary information is misleading, DART2 still asymptotically control FDR and maintain power at least as high as the BH procedure. We demonstrated the superior performance of DART2 through numerical studies under various settings and applied it to a gene association study, where we showed its superior accuracy and robustness under two different types of ancillary information. 

In our numerical studies, we also found that even if the ancillary information is misleading, it is still possible that DART2 can identify some additional true alternatives compared with the BH procedure. Thus, it is possible that we can further improve the power of DART2 by trying multiple random orderings. We will explore this direction in our future studies.

\newpage
\gdef\thesection{Appendix \Alph{section}}
\setcounter{section}{0}
\setcounter{subsection}{0}
\setcounter{equation}{0}
\setcounter{figure}{0}
\setcounter{table}{0}
\setcounter{page}{1}
\renewcommand\theequation{\arabic{equation}}

\renewcommand{\thepage}{A\arabic{page}}  
\renewcommand{\thetable}{A\arabic{table}}   
\renewcommand{\thefigure}{A\arabic{figure}}
\renewcommand{\theequation}{A\arabic{equation}}

\section{Simulation details}\label{app:simu}

\subsection{Simulated data}

We generated three simulation settings, each with n = 300 observations on m = 1000 hypotheses.
Before presenting the three settings, let's first define the notations that will be consistently used across all of them:
\begin{equation*}
\eta_{i}  =\Big\{\{[5.1\phi_1(d_{156,i})-0.9]\vee 0\}+4.5\phi_2(d_{800,i})-0.1\Big\}\vee 0;
\end{equation*}
where $\phi_1$ and $\phi_2$ are the PDF of
$\N(0,1)$ and $\N(0,0.1)$, respectively.

\begin{enumerate}[label=SE{\arabic*}:,align = left,leftmargin=*]
\item For node $i\in \{1,...,m\}$, the feature testing statistics $T_i$
  are independently generated from $\N(\sqrt{n}\theta_i,1)$, with
$\theta_{i}=
\frac{1}{5}\eta_{i}$.
\item Consider the linear mode
$$Y_i=\theta_{0,i}+\theta_{1,i}W_1+\theta_{2,i}W_2+\epsilon_i,\quad \text{with }\epsilon\sim N(0,1)$$
The covariates $W_{1}$ and $W_{2}$ are sampled from
  $\mathrm{Binom}(0.5)$ and $\mathrm{Unif}(0.1,0.5)$, respectively.
  Also let $\theta_{0,i}=\theta_{2,i}=0.1$ and
$\theta_{i}=\theta_{1,i}=\frac{1}{3}\eta_i$.
The feature test statistic $T_i$ is the Wald test statistics.
\item Consider the cox regression model
\begin{equation*}
\lambda_i(t)=\lambda_{0i}(t)\exp\{\theta_{1,i}W_{1}+\theta_{2,i}W_{2}\}
\end{equation*}
Where $\lambda_i(t)$ and $\lambda_{0i}(t)$ is the hazard and baseline hazard at time $t$, respectively. Set $\theta_{0,i}=\theta_{2,i}=0.1$ and
$\theta_i=\theta_{i}=\frac{1}{2}
\eta_{i}$. The covariates $W_{1}$ and $W_{2}$ are sampled from
$\mathrm{Binom}(0.5)$ and $\mathrm{Unif}(0.1,0.5)$, respectively. The
event time is generated from the exponential distribution with rate
$\exp\{\theta_{1,i}W_{1}+\theta_{2,i}W_{2}\}$, and the censoring time
is sampled from $\Unif(0,5)$. The feature test statistic $T_i$ is obtained
from the Wald test.
\end{enumerate}

\begin{figure}
  \begin{center}
    \includegraphics[width=\linewidth]{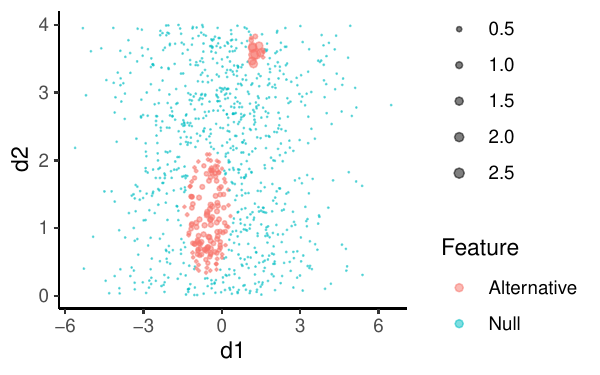}
  \end{center}
  \vspace{-2ex}
  \caption{Illustration of the simulated hypotheses' affiliated location and their corresponding $\eta_i$. Each hypothesis is represented by a dot. The dot color stands for the hypothesis status; and its size is $\text{L-eta}=\log(\eta_i+1)+0.01$, which is proportional to the signal strength.}
\label{appfig:eta}
\end{figure}

\subsection{Tunning parameter settings}
In our study, we applied various statistical methods to the three numerical settings including DART, DART2, BH procedure, the FDRL procedure, and AdaPT. The BH method does not require tuning parameters. For DART, following the tuning parameter selection criteria from Section 3.3 in \cite{li2023dart}, we constructed a 4-layer aggregation tree with $M=2$ and distance thresholds:
$g^{(2)}=1.33$, $g^{(3)}=1.56$, $g^{(4)}=1.90$. For RDART, following the similar parameter selection criteria, we constructed the aggregation tree up to 13 layers, with $M=2$ and distance thresholds:
$g^{(2)}=1.33$, $g^{(3)}=1.56$, $g^{(4)}=1.90$,$g^{(5)}=2.10$,
$g^{(6)}=2.60$,
$g^{(7)}=3.93$,
$g^{(8)}=4.13$,
$g^{(9)}=4.96$,
$g^{(10)}=5.40$,
$g^{(11)}=6.96$,
$g^{(12)}=9.58$,
$g^{(13)}=12.12$. For the FDR$_L$ methods, following recommendations from \cite{zhang2011multiple}, we set 
$k=5$ for our simulations, as it is suggested to use an odd number greater than three. \cite{zhang2011multiple} introduced two types of FDR$_L$ methods: FDR$_L$ I and FDR$_L$ II, but we only applied the FDR$_L$ II method in this paper given its consistent better performance compared to the FDR$_L$ I  \citep{zhang2011multiple,li2023dart}. With AdaPT, we followed the instructions found at \url{https://cran.r-project.org/web/packages/adaptMT/vignettes/adapt_demo.html} to set up its tuning parameters. 
During the simulation, we observed that AdaPT occasionally failed to produce results after extended processing times. To ensure the integrity of our findings, Figure~\ref{fig:methods} only includes data from repetitions where AdaPT successfully delivered results within 1 hour of processing. Table~\ref{tab:fails} summarizes the number of failed repetitions across different scenarios among 200 repetitions.

\begin{table}[ht]
\centering
\caption{Percentage of repetition fails to deliver testing result within 1 hour.}
\label{tab:fails}
\begin{tabular}{rrrrrrr}
  \hline
\multirow{2}{*}{SE}&\multicolumn{6}{c}{$\tau$}\\ \cline{2-7}
& 0 & 0.2 & 0.4 & 0.6 & 0.8 & 1 \\ 
  \hline
 1 & 17.0\% & 16.5\% & 8.0\% & 2.5\% & 3.0\% & 2.5\% \\ 
   2 & 10.5\% & 14.5\% & 14.0\% & 24.5\% & 33.0\% & 22.0\% \\ 
   3 & 24.5\% & 13.5\% & 1.5\% & 1.5\% & 0 & 1.5\% \\ 
   \hline
\end{tabular}
\end{table}

\subsection{Additional simulation outputs}
Figure~\ref{fig:methods_bar} displays the simulation results along with $90\%$ error bars to illustrate the variability of the outcomes across 200 repetitions.
\begin{figure}[!htpb]
\includegraphics[width=\textwidth]{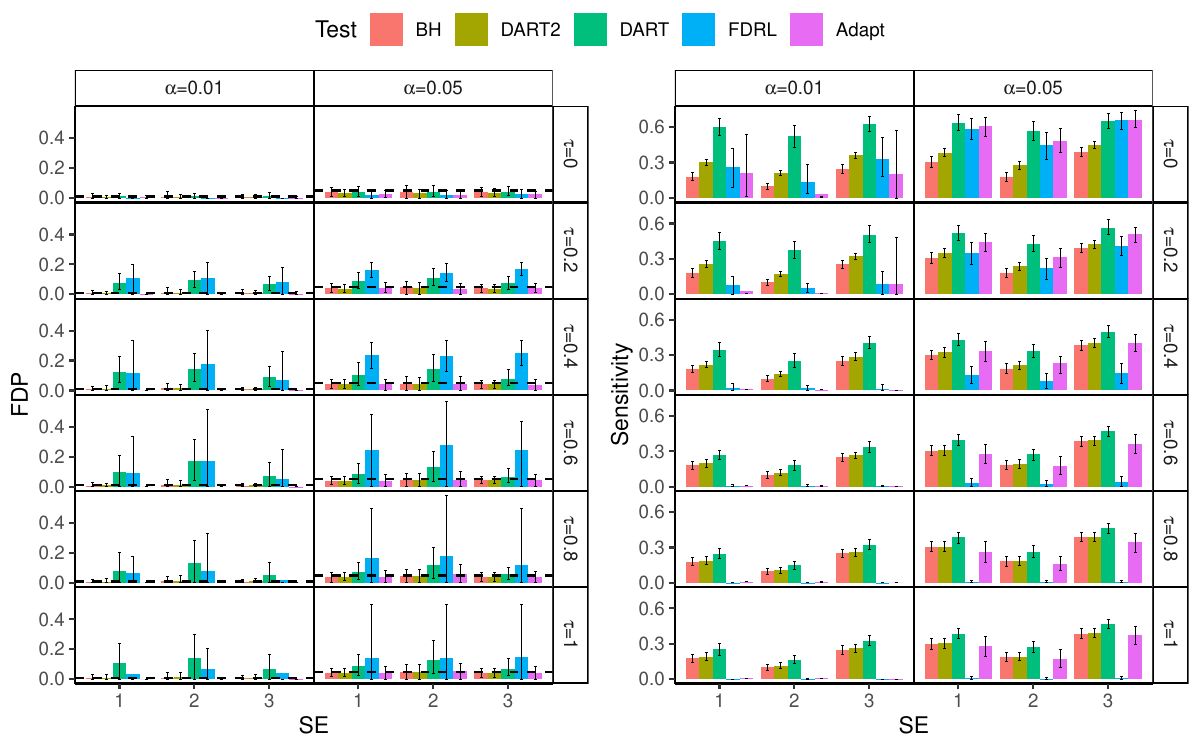}
    \caption{Performance comparison of the 7-layer DART2 with the competing method under different types of testing statistics, different desired FDR level $\alpha$ and different misleading level $\tau$. The primary bars indicate the average performance over $200$ repetitions.
    while the error bars represent the $90\%$ confidence intervals, denoting the $5\%$ and $95\%$ quantiles from $200$ simulations. 
    The left panel shows the average feature-level FDP, with dashed horizontal lines indicating the desired FDR $\alpha$. The right panel shows the average feature-level sensitivity.}
    \label{fig:methods_bar}
\end{figure}

\section{Proofs}\label{app:theory}

\subsection{Proofs of main theorems}\label{app:proof_thm}
Note that in the proof, we are proving a stronger version of the theorem 1 and 2, with $\alpha^\el=\alpha$, $\ell=1,\ldots,L$.\par
We introduce some notations before we provide the proofs. On layer $\ell$,
for a working node $S\in\mB^{(\ell)}$, let
$\mU(S)=\{S'\subset S:S'\in \cup_{\ell'=1}^{\ell-1}\mB^{(\ell')}\}$
be the collection of sets in the testing path of $S$.
In addition, let $\mU^c(S)=\{S''\in \cup_{\ell'=1}^{\ell-1}\mB^{(\ell')} :S''\cap S=\emptyset, S''\cup S\subset A,\text{ for some }A\in\mA^\el\}$ be the collection of sets that was planning to combined with $S$ on layer $\ell$ of the static aggregation tree but rejected on previous layers. When $S\in\mB^{(1)}$, we set $\mU(S)=\mU^c(S)=\emptyset$. 
We define $G_S(c)$ as the complementary CDF conditional on previous testing results. When $\ell=1$, we have $S=\{i\}\subset\{1,...,m\}$, and $G_S(c)=P(Z_i\geq c)$ with $Z_1,\ldots,Z_m\overset{iid}{\sim} N(0,1)$. When $\ell>1$, the oracle rejection path for set $S\in\mB^\el$ is recursively defined as \[\mQ_z^{(1:\ell-1)}=\{z:\forall S'\in\mU(S), G_{S'}(Z_{S'})\geq\hat t^{(\ell_{S'})}(\alpha),\forall S''\in\mU^c(S), G_{S''}(Z_{S''})\leq\hat t^{(\ell_{S''})}(\alpha)\},\] where
\[
G_S(c)=\P\big(Z_S\geq c\big|\mQ_z^{(1:\ell-1)}\big)
\]
and $Z_S=\sum_{i\in S} Z_i/\sqrt{|S|}$, and $\ell_{S'},\ell_{S''}\in\{1,...,\ell-1\}$ is the value s.t. $S'\in\mB^{(\ell_{S'})}$ and $S''\in\mB^{(\ell_{S''})}$, respectively. \par
Given $Z_1,\ldots,Z_m$ are mutually independent, we have
\[G_S(c)=\P\big(Z_S\geq c\big|\forall S'\in\mU(S), G_{S'}(Z_{S'})\geq\hat t^{(\ell_{S'})}(\alpha)\big)
\]
Given the definition of $G_S(c)$, we define the rejection path as
\begin{equation}
\label{eq::rejp}
\mQ^{(1:\ell-1)}=\{t:\forall S'\in\mU(S), G_{S'}(T_{S'})\geq\hat t^{(\ell_{S'})}(\alpha),\forall S''\in\mU^c(S), G_{S''}(T_{S''})\leq\hat t^{(\ell_{S''})}(\alpha)\}
\end{equation}
\par
In addition, for two sequence of real numbers $a_m$ and $b_m$, we write $a_m=o(b_m)$ when $a_m/b_m\to 0$, and $a_m=O(b_m)$ when $\lim_{m\to\infty}|a_m/b_m|\leq C$ for some constant $C$.
To prove the asymptotic properties of RDART, we need the following lemmas. Here, Lemma~\ref{lem::mcond} is introduced and proved in DART, Lemma 3 \citep{li2023dart}.

\mathleft
\begin{lem}\label{lem::mcond}
Let $\tilde\Omega_0=\{i: \tilde T_i \text{ follows } \Unif(0,1) \}$, $\mB_{0a}^{(\ell)}:=\{S\in\mB_0^\el:\exists A\in\mA^{(L)}\setminus \mA', s.t. S\subset A\}$,
and  $\mB_{0b}^{(\ell)}:=\{S\in\mB_{0a}^\el:S\in \tilde\Omega_0\}$,
we have:
\begin{flalign*}
\text{(1)}&\quad \max_{S\in\mB_{0a}^{(\ell)}}\sup_{c\in[0,\gamma_m]}\bigg|\frac{G_S(c)}{\bar\Phi(c)}-1\bigg|\to 0\\
\text{(2)}&
\quad \max_{S\in\mB_{0b}^{(\ell)}}\sup_{c\in[0,\bar\Phi^{-1}(1/m)]}\bigg|\frac{\P(T_S>c|\mQ^{(1:\ell-1)})}{\P(T_S>c)}-1\bigg|\to 0
\end{flalign*}
\end{lem}
\mathcenter

\begin{lem}
\label{lem::tail}
Let $ \quad \gamma_m=\sqrt{2\log m+\log\log\log m}$, then
\[\P(\hat c^{(\ell)}<\gamma_m)\to 1,\forall \ell\in\{1,\ldots,L\}\]
\end{lem}

\begin{lem}
\label{lem::nullsig}
For any node $S\in\mB^\el$ with $S\subset \Omega_0$,
\begin{equation}
\label{eq::nullsig}
P(\exists i,j, \quad s.t. \quad T_i\geq \hat t^*_S, T_j\geq \hat t^*_S|T_S\geq \hat c^\el)\to 0
\end{equation}
\end{lem}

\begin{proof}[\textbf{Proof of Theorem~\ref{thm:indfdr}}] 
Since $FDP$ is a random variable bounded by $1$, to prove the $FDR$ control, it is suffice to show:
$$\lim_{m\to\infty}\P(FDP\leq \alpha+\epsilon)=1.$$
Let $R^\el=\{j\in \Omega: \exists S \text{ s.t. }  S\in\mB^{(\ell)}\cap R_\node \text{ and } j\in S, T_j\geq \hat t_S\}$ be the set of rejected hypotheses on layer $\ell$, 
$V^{(\ell)}=\{i\in R^\el: \exists S \text{ s.t. }  S\in\mB_0^{(\ell)} \text{ and }i\in S\}$ and
$W^{(\ell)}=\{i\in R^\el: \exists S \text{ s.t. }  S\in\mB_1^{(\ell)} \text{ and }i\in S\}$ be the rejected hypothesis who were originally selected from the node-level null node and alternative node on
layer $\ell$, respectively. On layer $\ell$, since $m_1=O(m^{r_1})$, we have
\begin{align}
\P(\forall j
\in W^\el, j\in \Omega_1)\geq & 1-m^{r_1} \P( T_j\geq \beta/\sqrt{M^{L-1}},j\in\Omega_0)\nonumber\\
\geq & 1- o(m^{r_1+(r_1-1)/M^{L-1}})\nonumber\\
\to & 1 \label{eq:key}
\end{align}
Thus, if
\begin{align}
\label{eq:layw}
\lim_{m\to\infty}\P(FDP^{(\ell)}\leq \alpha+\epsilon)=1,
\end{align}
then together with Lemma~\ref{lem::nullsig}, we have 
\begin{align*}
\P\bigg(\frac{|R^\el\cap \Omega_0|}{|R^\el|}\leq \alpha+\epsilon\bigg)&\geq\P\bigg(\frac{|W^\el\cap\Omega_0|+|V^\el|}{|R^\el|}\leq \alpha+\epsilon\bigg)\\
&\geq P(FDP^\el\leq \alpha+\epsilon)\to 1
\end{align*}
and,
\begin{align*}
\lim_{m\to\infty}\P(FDP\leq \alpha+\epsilon) \leq & \lim_{m\to\infty}\P\bigg(\max_{\ell}\frac{|R^\el\cap \Omega_0|}{|R^\el|}\leq \alpha+\epsilon\bigg)\\
\leq & \lim_{m\to\infty}\P\bigg(\frac{|R^\el\cap \Omega_0|}{|R^\el|}\leq \alpha+\epsilon,\forall \ell\in\{1,\ldots,L\}\bigg)\\
\to & 1
\end{align*}
Thus, it is suffice to prove \eqref{eq:layw} for any layer $
\ell=1,\ldots,L$.
\par

The random variable $FDP^{(\ell)}$ can be decomposed to the product into two parts.
\begin{align}
\label{FDPparts}
FDP^{(\ell)}=&\frac{\sumnull\csil I\{T_S>\hat c^{(\ell)}\}}{\sumnull\csil G(\hat c^{(\ell)})}\times \frac{\sumnull\csil G(\hat c^{(\ell)})}{\max(\sum_{S\in \mB^{(\ell)}}|S|I\{T_S>\hat c^{(\ell)}\},1)}
\end{align}
Based on \eqref{FDPparts}, in order to prove
$\lim_{m\to\infty}\P(FDP^{(\ell)}\leq \alpha+\epsilon)=1$
for all $\epsilon>0$, we only need prove
\begin{align}
&\lim_{m\to\infty}\P\Bigg\{\frac{\sumnull\csil I\{T_S>\hat c^{(\ell)}\}}{\sumnull\csil G(\hat c^{(\ell)})}-1<\epsilon\Bigg\}\to 1\label{part1}\\
&\lim_{m\to\infty}\P\Bigg\{\bigg|\frac{\sumnull\csil G(\hat c^{(\ell)})}{\max(\sum_{S\in \mB^{(\ell)}}|S|I\{T_S>\hat c^{(\ell)}\},1)}-\alpha\bigg|>\epsilon\Bigg\}\to 0\label{part2}
\end{align}
\eqref{part2} is immediately followed by the continuity of $G(\cdot)$ and the monotonicity of the indicator function. We will prove \eqref{part1} by induction. 
\par
\textbf{Layer 1:}\par 
Let $\nu_m=1/\log m$ and $J=\J$. Define a sequence $0=c_0<...<c_{J}=\gamma_m$ satisfies $c_k-c_{k-1}=\nu_m$ for $1\leq k<J$ and $c_{J}-c_{J-1}\leq \nu_m$. By Markov inequality, we have \begin{equation}
\label{eq::int}
\int_0^{c'}P\bigg\{\bigg|\frac{\sum_{S\in\mB_{0}^{(1)}}I(T_S>c)-P(T_S>c)(1-\delta_{0m})}{\sum_{S\in\mB_{0}^{(1)}}\bar\Phi(c)}\bigg|\geq \epsilon\bigg\}dc=o(\nu_m)
\end{equation}
Accordingly, $\forall \epsilon>0$,
\begin{align}
&P\bigg(\max_{0\leq k\leq J}\bigg|\frac{\sum_{i\in\mB_{0}^{(1)}}I(T_i >c_k)-\sum_{i\in\mB_{0}^{(1)}}P(T_i>c_k)(1-\delta_{0m})
}{\sum_{i\in\mB_{0}^{(1)}} G(c_k)}\bigg|>\epsilon
\bigg)
\to 0\label{L1.1.2}
\end{align}
 Together with the fact that $\sup_{j=1,...,J}\bigg|G(k)/G(k-1)-1\bigg|=o(1)$, we have 
\begin{align}
\label{eq:px1}
P\bigg(\sup_{c\in [0,\gamma_m]}\bigg|\frac{\sum_{i\in\mB_0^{(1)}}I(X_i>  c)-m_0 G(c)
}{m_0 G(c)}\bigg|>\epsilon\bigg)=o(1)
\end{align}
Thus,
\begin{equation}
\label{eq:fdr1}
\P(FDP^{(1)}<\alpha+\epsilon)\to 1
\end{equation}

Let \[\mX^{(1)}=\bigg\{x:\bigg|\frac{\sum_{i\in\mB_0^{(1)}}I(T_i>  \hat c^{(1)})-m_0 G(\hat c^{(1)})
}{m_0 G(\hat c^{(1)})}\bigg|\leq\epsilon\bigg\}\]


On $\mX^{(1)}$, 
\begin{equation*}
\sumnull I(T_S>\hat c^{(1)})\leq 
m_0 G(\hat c^{(1)})+m_0 G(\hat c^{(1)})\epsilon
\end{equation*}
Combined with
\begin{align*}
m_0 G(\hat c^{(1)})&\leq \alpha\sum_{S\in \mB^{(\ell)}} \bI\{T_S>\hat c^{(1)}\}\leq m_0\alpha\bI\{T_S>\hat c^{(1)}\}+\alpha Cm^{r_1}
\end{align*}
we have,
\begin{equation*}
(1-\alpha-\alpha\epsilon)m_0G(\hat c^{(1)})\leq \alpha Cm^{r_1}
\end{equation*}
and accordingly $\hat c^{(1)}\geq \beta_m$ when $m$ large enough. Thus, together with \eqref{eq:px1}, we have
\[\P(\hat c^{(1)}\geq \beta_m)\to 1\]

\textbf{Layer $\ell\geq 2$:}\par

Suppose on layer $h=1,\ldots,\ell-1$, $\P(\hat c^{(h)}\geq\beta_m)\to 1$ and $FDP^{(h)}\leq\alpha/M^{h-1}$ with probability converging to 1, then together with lemma~\ref{lem::mcond}, similar to the proof in Layer 1, we have
\begin{align*}
P\bigg(\sup_{c\in[0,\gamma_m]}\frac{\sum_{S\in\mB_0^{(\ell)}}|S|I(T_S>c)-\sum_{S\in\mB_0^{(\ell)}}|S|G(c)
}{\sum_{S\in\mB_0^{(\ell)}}|S|G(c)}>\epsilon\Bigg|\mQ^{(1:\ell-1)}\bigg)=o(1)
\end{align*}
Define $\mX^{(\ell)}=\bigg\{x:\bigg|\frac{\sum_{S\in\mB_0^{(\ell)}}|S|I(T_S>  \hat c^{(\ell)})-\sum_{S\in\mB_0^{(\ell)}}|S|G(\hat c^{(\ell)})
}{\sum_{S\in\mB_0^{(\ell)}}|S|G(\hat c^{(\ell)})}\bigg|\leq\epsilon\bigg\}$, then
\[\P(\mX^{(\ell)})\to 1.\]
Thus, 
\[\P(FDP^{(\ell)}<\alpha+\epsilon)\to 1\]

In addition, on $\cap_{t=1}^\ell \mX^{(t)}$, given $\P(FDP^{(h)}<\alpha+\epsilon)\to 1$ for all $h=1,\ldots,\ell$, we have
\[\P(\hat c^{(\ell)}\geq \beta_m)\to 1\]

\end{proof}

\begin{proof}[\textbf{Proof of Theorem~\ref{thm:robfdr}}]
Let $\hat t_S$ represent the robust refining threshold and $\hat t^*_S$ represent the naive refining threshold. Based on definition~\ref{def:rrt} for robust refining threshold, we have
\[\hat t_S\geq \hat t^*_S\]
Then, by redefining the $V^\el$, $W^\el$ and $R^\el$ based on the robust refine threshold, and considering that the robust refining threshold guarantees that at least one hypothesis will be rejected in every screened-out node, we can still get.
\[\P(\forall j\in W^\el,j\in\Omega_1)\to 1\]

Then following the same proving process in Theorem~\ref{thm:indfdr}, we can prove the asymptotic FDR control under the robust refining threshold.
\end{proof}

\subsection{Proofs of Lemmas}\label{app:proof_lem}

\begin{proof}[\textbf{Proof of Lemma~\ref{lem::tail}}]
\begin{align*}
\P(\hat c^\el<\gamma_m)&\geq \P\bigg(\frac{m^\el o(m^{-1})}{\max\{\sum_{S\in \mB^l}I(T_S>\gamma_m),1\}}\leq \alpha\bigg|\mQ^{(1:\el-1)}\bigg)\\
&\geq \P(o(m^{-1})m\leq \alpha)\\
&\to 1
\end{align*}

\end{proof}

\begin{proof}[\textbf{Proof of Lemma~\ref{lem::nullsig}}]
It is equivalent to prove 
\[\sup_{ a\in[\beta,\gamma]}\frac{P\big(Z_1\geq a/\sqrt{|S|},Z_2\geq a/\sqrt{|S|},Z_1+Z_2+\sqrt{|S|-2}Z_3\geq\sqrt{|S|}a\big)}{\bar\Phi(a)}\to 0\]

Since for $A_S=\big[\sqrt{(|S|-2)/|S|}a-\log\log m,+\infty\big]$,
\[
\sup_{a\in[\beta,\gamma]}\frac{P\big(Z_1\geq a/\sqrt{|S|},Z_2\geq a/\sqrt{|S|},Z_1+Z_2+\sqrt{|S|-2}Z_3\geq\sqrt{|S|}a,Z_3\in A_S\big)}{\bar\Phi(a)}\to 0
\]
It is sufficient to show
\[\sup_{ a\in[\beta,\gamma]}H(a)\to 0\]
where
\[
H(a)=\frac{P\big(Z_1\geq a/\sqrt{|S|},Z_2\geq a/\sqrt{|S|},Z_1+Z_2+\sqrt{|S|-2}Z_3\geq\sqrt{|S|}a,Z_3\in A_S^c\big)}{\bar\Phi(a)}
\]
\par
Together with L'Hospital rule and some tedious calculation, we have 

\begin{equation}
\label{eq:zzz}
\lim_{m\to\infty}\frac{\frac{d}{da}P\big(Z_1\geq a/\sqrt{|S|},Z_2\geq a/\sqrt{|S|},Z_1+Z_2+\sqrt{|S|-2}Z_3\geq\sqrt{|S|}a,Z_3\in A_S^c\big)}{\frac{d}{da}\bar\Phi(a)}=0
\end{equation}
Hence, whenever $a=\beta$ or $\gamma$,
\begin{align}
\lim_{m\to \infty}H(a)=&\lim_{m\to\infty}\frac{\frac{d}{da}P\big(Z_1\geq a/\sqrt{|S|},Z_2\geq a/\sqrt{|S|},Z_1+Z_2+\sqrt{|S|-2}Z_3\geq\sqrt{|S|}a,Z_3\in A_S^c\big)}{\frac{d}{da}\bar\Phi(a)}\nonumber\\
=&0\label{eq:bd}
\end{align}
Define $\mH_m:=\{a\in[\beta,\gamma]:\frac{d}{d a}H(a)=0\}$, we have
\begin{align*}
H(a)=\frac{\frac{d}{da}P\big(Z_1\geq a/\sqrt{|S|},Z_2\geq a/\sqrt{|S|},Z_1+Z_2+\sqrt{|S|-2}Z_3\geq\sqrt{|S|}a,Z_3\in A_S^c\big)}{\frac{d}{da}\bar\Phi(a)}
\end{align*}
and hence 
\begin{equation}
\label{eq:sup}
\lim \sup_{a\in\mH_m}H(a)\to 0
\end{equation}
So together with \eqref{eq:bd} and \eqref{eq:sup}, we have
\begin{equation*}
\sup_{ a\in[\beta,\gamma]}H(a)=\max\bigg\{H(\beta),H(\gamma),\sup_{a\in \mH_m}H(a)\bigg\}\to 0
\end{equation*}

\end{proof}

\bibliographystyle{elsarticle-harv}
\bibliography{manu.bib}

\begin{thebibliography}{15}
\expandafter\ifx\csname natexlab\endcsname\relax\def\natexlab#1{#1}\fi
\providecommand{\url}[1]{\texttt{#1}}
\providecommand{\href}[2]{#2}
\providecommand{\path}[1]{#1}
\providecommand{\DOIprefix}{doi:}
\providecommand{\ArXivprefix}{arXiv:}
\providecommand{\URLprefix}{URL: }
\providecommand{\Pubmedprefix}{pmid:}
\providecommand{\doi}[1]{\href{http://dx.doi.org/#1}{\path{#1}}}
\providecommand{\Pubmed}[1]{\href{pmid:#1}{\path{#1}}}
\providecommand{\bibinfo}[2]{#2}
\ifx\xfnm\relax \def\xfnm[#1]{\unskip,\space#1}\fi
\bibitem[{Barber and Ramdas(2017)}]{barber2017p}
\bibinfo{author}{Barber, R.F.}, \bibinfo{author}{Ramdas, A.},
  \bibinfo{year}{2017}.
\newblock \bibinfo{title}{The p-filter: multilayer false discovery rate control
  for grouped hypotheses}.
\newblock \bibinfo{journal}{Journal of the Royal Statistical Society Series B:
  Statistical Methodology} \bibinfo{volume}{79}, \bibinfo{pages}{1247--1268}.
\bibitem[{Benjamini and Hochberg(1995)}]{benjamini1995controlling}
\bibinfo{author}{Benjamini, Y.}, \bibinfo{author}{Hochberg, Y.},
  \bibinfo{year}{1995}.
\newblock \bibinfo{title}{Controlling the false discovery rate: a practical and
  powerful approach to multiple testing}.
\newblock \bibinfo{journal}{Journal of the Royal statistical society: series B
  (Methodological)} \bibinfo{volume}{57}, \bibinfo{pages}{289--300}.
\bibitem[{Cai et~al.(2022)Cai, Sun and Xia}]{cai2022laws}
\bibinfo{author}{Cai, T.T.}, \bibinfo{author}{Sun, W.}, \bibinfo{author}{Xia,
  Y.}, \bibinfo{year}{2022}.
\newblock \bibinfo{title}{Laws: A locally adaptive weighting and screening
  approach to spatial multiple testing}.
\newblock \bibinfo{journal}{Journal of the American Statistical Association}
  \bibinfo{volume}{117}, \bibinfo{pages}{1370--1383}.
\bibitem[{Davis and Meltzer(2007)}]{davis2007geoquery}
\bibinfo{author}{Davis, S.}, \bibinfo{author}{Meltzer, P.S.},
  \bibinfo{year}{2007}.
\newblock \bibinfo{title}{Geoquery: a bridge between the gene expression
  omnibus (geo) and bioconductor}.
\newblock \bibinfo{journal}{Bioinformatics} \bibinfo{volume}{23},
  \bibinfo{pages}{1846--1847}.
\bibitem[{Hu et~al.(2010)Hu, Zhao and Zhou}]{hu2010false}
\bibinfo{author}{Hu, J.X.}, \bibinfo{author}{Zhao, H.}, \bibinfo{author}{Zhou,
  H.H.}, \bibinfo{year}{2010}.
\newblock \bibinfo{title}{False discovery rate control with groups}.
\newblock \bibinfo{journal}{Journal of the American Statistical Association}
  \bibinfo{volume}{105}, \bibinfo{pages}{1215--1227}.
\bibitem[{Lei and Fithian(2018)}]{lei2018adapt}
\bibinfo{author}{Lei, L.}, \bibinfo{author}{Fithian, W.}, \bibinfo{year}{2018}.
\newblock \bibinfo{title}{Adapt: an interactive procedure for multiple testing
  with side information}.
\newblock \bibinfo{journal}{Journal of the Royal Statistical Society: Series B}
  \bibinfo{volume}{80}, \bibinfo{pages}{649--679}.
\bibitem[{Leung and Sun(2022)}]{leung2022zap}
\bibinfo{author}{Leung, D.}, \bibinfo{author}{Sun, W.}, \bibinfo{year}{2022}.
\newblock \bibinfo{title}{Zap: z-value adaptive procedures for false discovery
  rate control with side information}.
\newblock \bibinfo{journal}{Journal of the Royal Statistical Society Series B:
  Statistical Methodology} \bibinfo{volume}{84}, \bibinfo{pages}{1886--1946}.
\bibitem[{Li et~al.(2023)Li, Sung and Xie}]{li2023dart}
\bibinfo{author}{Li, X.}, \bibinfo{author}{Sung, A.D.}, \bibinfo{author}{Xie,
  J.}, \bibinfo{year}{2023}.
\newblock \bibinfo{title}{Dart: Distance assisted recursive testing}.
\newblock \bibinfo{journal}{Journal of Machine Learning Research}
  \bibinfo{volume}{24}, \bibinfo{pages}{1--41}.
\bibitem[{Liu et~al.(2013)}]{liu2013gaussian}
\bibinfo{author}{Liu, W.}, et~al., \bibinfo{year}{2013}.
\newblock \bibinfo{title}{Gaussian graphical model estimation with false
  discovery rate control}.
\newblock \bibinfo{journal}{The Annals of Statistics} \bibinfo{volume}{41},
  \bibinfo{pages}{2948--2978}.
\bibitem[{Qiu et~al.(2021)Qiu, Murrugarra-Llerena, Silva, Lin and
  Chinchilli}]{qiu2021neurt}
\bibinfo{author}{Qiu, L.}, \bibinfo{author}{Murrugarra-Llerena, N.},
  \bibinfo{author}{Silva, V.}, \bibinfo{author}{Lin, L.},
  \bibinfo{author}{Chinchilli, V.M.}, \bibinfo{year}{2021}.
\newblock \bibinfo{title}{Neurt-fdr: Controlling fdr by incorporating feature
  hierarchy}.
\newblock \bibinfo{journal}{arXiv preprint arXiv:2101.09809} .
\bibitem[{Stouffer et~al.(1949)Stouffer, Suchman, DeVinney, Star and
  Williams~Jr}]{stouffer1949american}
\bibinfo{author}{Stouffer, S.A.}, \bibinfo{author}{Suchman, E.A.},
  \bibinfo{author}{DeVinney, L.C.}, \bibinfo{author}{Star, S.A.},
  \bibinfo{author}{Williams~Jr, R.M.}, \bibinfo{year}{1949}.
\newblock \bibinfo{title}{The american soldier: Adjustment during army
  life.(studies in social psychology in world war ii), vol. 1} .
\bibitem[{Xie and Li(2018)}]{xie2018false}
\bibinfo{author}{Xie, J.}, \bibinfo{author}{Li, R.}, \bibinfo{year}{2018}.
\newblock \bibinfo{title}{False discovery rate control for high dimensional
  networks of quantile associations conditioning on covariates}.
\newblock \bibinfo{journal}{J R Stat Soc Series B Stat Methodol}
  \bibinfo{volume}{80}, \bibinfo{pages}{1015--1034}.
\newblock \DOIprefix\doi{10.1111/rssb.12288}.
\bibitem[{Yang et~al.(2024)Yang, Wang and Chen}]{yang20242dgbh}
\bibinfo{author}{Yang, L.}, \bibinfo{author}{Wang, P.}, \bibinfo{author}{Chen,
  J.}, \bibinfo{year}{2024}.
\newblock \bibinfo{title}{2dgbh: Two-dimensional group benjamini--hochberg
  procedure for false discovery rate control in two-way multiple testing of
  genomic data}.
\newblock \bibinfo{journal}{Bioinformatics} \bibinfo{volume}{40},
  \bibinfo{pages}{btae035}.
\bibitem[{Yun et~al.(2022)Yun, Zhang and Li}]{yun2022detection}
\bibinfo{author}{Yun, S.}, \bibinfo{author}{Zhang, X.}, \bibinfo{author}{Li,
  B.}, \bibinfo{year}{2022}.
\newblock \bibinfo{title}{Detection of local differences in spatial
  characteristics between two spatiotemporal random fields}.
\newblock \bibinfo{journal}{Journal of the American Statistical Association}
  \bibinfo{volume}{117}, \bibinfo{pages}{291--306}.
\bibitem[{Zhang et~al.(2011)Zhang, Fan and Yu}]{zhang2011multiple}
\bibinfo{author}{Zhang, C.}, \bibinfo{author}{Fan, J.}, \bibinfo{author}{Yu,
  T.}, \bibinfo{year}{2011}.
\newblock \bibinfo{title}{Multiple testing via {FDR$_L$} for large scale
  imaging data}.
\newblock \bibinfo{journal}{Annals of Statistics} \bibinfo{volume}{39},
  \bibinfo{pages}{613}.

\end{thebibliography}



\end{document}